\documentclass[jair,twoside,11pt,theapa]{article}
\usepackage{jair, theapa, rawfonts}

\usepackage{amsmath, amssymb, amsthm}
\usepackage{mathtools}
\usepackage{tikz}
\usetikzlibrary{arrows.meta, shapes}
\tikzset{bigneuron/.style={circle, draw, inner sep = 0, minimum width = 5.5ex}}
\tikzset{smallneuron/.style={circle, draw, inner sep = 0, minimum width = 4ex}}
\tikzset{transform/.style={fill=white, circle}}
\tikzset{connection/.style={-{Stealth}}}

\usepackage[capitalize, nameinlink]{cleveref}
\usepackage{booktabs}
\usepackage{caption}

\newcommand{\relu}{
	\begin{tikzpicture}
		\draw [line width=1pt] (-1.1ex,0) -- (0,0) -- (0.9ex,0.9ex);
	\end{tikzpicture}
      }

\newcommand{\abs}[1]{\left\lvert #1 \right\rvert}
\newcommand{\R}{\mathbb{R}}
\newcommand{\N}{\mathbb{N}}
\newcommand{\Q}{\mathbb{Q}}
\DeclareMathOperator{\poly}{poly}
\DeclareMathOperator{\dist}{dist}

\newcommand{\problemdef}[3]{
	\begin{center}
		\begin{minipage}{0.95\linewidth}
			\noindent
			\textsc{#1}
			
			\vspace{2pt}
			\setlength{\tabcolsep}{3pt}
			\begin{tabular}{p{0.19\linewidth}p{0.75\linewidth}}
				\textbf{Input:} 		& #2 \\
				\textbf{Question:} 	& #3
			\end{tabular}
		\end{minipage}
	\end{center}
}

\newcommand{\optproblemdef}[3]{
	\begin{center}
		\begin{minipage}{0.95\linewidth}
			\noindent
			\textsc{#1}
			
			\vspace{2pt}
			\setlength{\tabcolsep}{3pt}
			\begin{tabular}{p{0.12\linewidth}p{0.82\linewidth}}
				\textbf{Input:} 		& #2 \\
				\textbf{Task:} 	& #3
			\end{tabular}
		\end{minipage}
	\end{center}
}

\newtheorem{theorem}{Theorem}
\newtheorem{proposition}[theorem]{Proposition}
\newtheorem{corollary}[theorem]{Corollary}
\Crefname{theorem}{Theorem}{Theorems}
\Crefname{proposition}{Proposition}{Propositions}
\Crefname{corollary}{Corollary}{Corollaries}

\jairheading{74}{2022}{1775-1790}{12/2021}{08/2022}
\ShortHeadings{Complexity of ReLU Training Parameterized by Dimensionality}
{Froese, Hertrich \& Niedermeier}
\firstpageno{1775}

\begin{document}

\title{The Computational Complexity of ReLU Network Training\\Parameterized by Data Dimensionality}

\author{\name Vincent Froese \email vincent.froese@tu-berlin.de \\
  \addr Technische Universit{\"a}t Berlin,\\
  Algorithmics and Computational Complexity,\\
  Ernst-Reuter-Platz 7,\\
  D-10587 Berlin, Germany
  \AND
       \name Christoph Hertrich \email c.hertrich@lse.ac.uk\\
       \addr London School of Economics and Political Science,\\
       Department of Mathematics,\\
       Houghton Street,\\
       London WC2A 2AE, United Kingdom
       \AND
       \name Rolf Niedermeier \email rolf.niedermeier@tu-berlin.de\\
       \addr Technische Universit{\"a}t Berlin,\\
  Algorithmics and Computational Complexity,\\
  Ernst-Reuter-Platz 7,\\
  D-10587 Berlin, Germany
     }


\maketitle

\begin{abstract}
Understanding the computational complexity of training simple 
	neural networks with rectified linear units (ReLUs)
	has recently been a subject of intensive research.
	Closing gaps and complementing results from the literature,
	we present several results on the parameterized complexity 
	of training two-layer ReLU networks with respect to various loss functions.
	After a brief discussion of other parameters, we focus on 
	analyzing the influence of the dimension~$d$ of the training data on the computational complexity.
	We provide running time lower bounds in terms of~\mbox{W[1]-hardness} for parameter~$d$ and prove that known
	brute-force strategies are essentially optimal 
	(assuming the Exponential Time Hypothesis).
	In comparison with previous work, our results hold for a broad(er) 
	range of loss functions, including $\ell^p$-loss for all $p\in[0,\infty]$.
	In particular, 
	we improve a known polynomial-time algorithm for constant~$d$ 
	and convex loss functions to a more general class of loss functions, 
	matching our running time lower bounds also in these cases.
\end{abstract}

\section{Introduction}

Dimensionality reduction of data is a central issue in many machine
learning scenarios~\cite{BFN19,MPH09}.
In this paper, our focus is on addressing a natural follow-up 
question: To what extent can ``low-dimensionality'' of
data points help in lowering the worst-case 
computational complexity of the task of 
training 
neural networks?
This question is particularly relevant from a practical point of view
because real-life data, even if high-dimensional, is often assumed to be contained in a low-dimensional
submanifold of the input space.
To answer this question, we will employ tools and 
concepts from parameterized complexity analysis.
Doing so, we focus on the very basic case of 
two-layer ReLU (rectified linear units) neural network training.
We believe that the studied problem, though very simple and practically irrelevant, is a basic building block in the grand challenge of gaining a fundamental understanding of the power and limitations of ReLU networks.
As already pointed out in the literature~\cite{BJW19,GKMR21}, understanding shallow networks is
a natural first step which is already quite involved and challenging.
To the best of our knowledge, we are the first to apply principles from parameterized complexity theory
to empirical risk minimization of ReLU networks.
Before proceeding with a discussion of related work
and our new contributions, we first provide some formal
definitions concerning the problems which are central to our work.

\begin{figure}[t]
	\centering
	\begin{tikzpicture}[]
		\footnotesize
		\node[smallneuron, label=below:{$b_1$}] (n1) at (0,11ex) {\relu};
		\node[smallneuron, label=below:{$b_2$}] (n2) at (0,2ex) {\relu};
		\node[rotate=90] (npunkt) at (0,-6ex) {$\mathbf{\cdots}$};
		\node[smallneuron, label=below:{$b_k$}] (n3) at (0,-11ex) {\relu};
		\node[smallneuron, label=left:{$x_1$}] (in1) at (-24ex,7.5ex) {};
		\node[smallneuron, label=left:{$x_2$}] (in2) at (-24ex,2.5ex) {};
		\node[rotate=90] (inpunkt) at (-24ex,-2.5ex) {$\mathbf{\cdots}$};
		\node[smallneuron, label=left:{$x_d$}] (in3) at (-24ex,-7.5ex) {};
		\node[smallneuron] (out) at (24ex,0) {};
		\draw[connection] (in1) -- (n1) node[above, pos=0.8] {$\mathbf w_1$};
		\draw[connection] (in2) -- (n1);
		\draw[connection] (in3) -- (n1);
		\draw[connection] (in1) -- (n2);
		\draw[connection] (in2) -- (n2);
		\draw[connection] (in3) -- (n2) node[below, pos=0.8] {$\mathbf w_2$};
		\draw[connection] (in1) -- (n3);
		\draw[connection] (in2) -- (n3);
		\draw[connection] (in3) -- (n3) node[below, pos=0.8] {$\mathbf w_k$};
		\draw[connection] (n1) -- (out) node[above, pos=0.75] {$a_1$};
		\draw[connection] (n2) -- (out) node[below, pos=0.75] {$a_2$};
		\draw[connection] (n3) -- (out) node[below, pos=0.75] {$a_k$};
	\end{tikzpicture}
	\caption{The neural network architecture we study in this paper: After the input layer (left) with $d$ input neurons, we have one hidden layer with $k$ ReLU neurons and a single output neuron without additional activation function.}
	\label{Fig:NNArchitecture}
\end{figure}
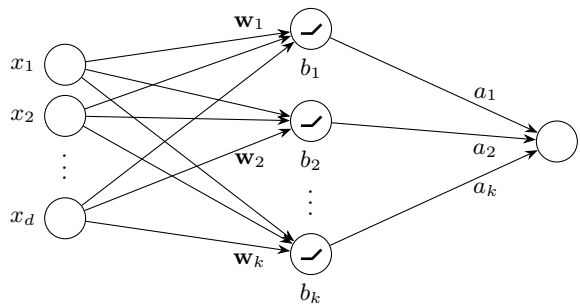

We study empirical risk minimization for neural networks with ReLU activation.
A \emph{rectifier} is the function~$[x]_+ \coloneqq \max(0,x)$.
Given a loss function $\ell\colon\R\times\R\to\R_{\geq0}$, mapping the predicted and the true label to a loss value, the problem of training a two-layer ReLU neural network with~$k$ hidden neurons and a single output neuron (see \Cref{Fig:NNArchitecture}) is defined as follows:

\optproblemdef{$k$-ReLU($\ell$)}{Data points~$(\mathbf{x}_1,y_1),\ldots,(\mathbf{x}_n,y_n)\in \R^d\times\R$.}{Find weight vectors $\mathbf w_1, \ldots, \mathbf w_k\in \R^d$, biases $b_1,\ldots,b_k\in\R$, and coefficients $a_1,\ldots,a_k\in\{-1,1\}$ that minimize
	\[
	\sum_{i=1}^n \ell\Bigg(\sum_{j=1}^k a_j[\langle\mathbf{w}_j,\mathbf{x}_i\rangle + b_j]_+,\ y_i\Bigg).
	\]
      }
As usual in the context of computational complexity analysis, we consider the decision instead of the optimization version. The corresponding decision problem is
to decide whether a target training error of at most~$\gamma\in\R$ can be achieved.
We believe it is of fundamental interest to know where the borderlines of exact solvability of ReLU training are.
One of our main contributions is to prove strong computational hardness results that already hold for only a single hidden ReLU neuron, that is, $k=1$.

In this work, we focus on the case of $\ell^p$-loss functions, that is, $\ell(\hat{y}, y)=\abs{\hat{y} - y}^p$, for \mbox{$p\in\left]0,\infty\right[$}. Note that this includes both convex and concave loss functions (with respect to the absolute error~$\abs{\hat{y}-y}$). In addition, we also investigate the limit cases $p=0$ and $p=\infty$. The $\ell^0$-loss (which is completely insensitive to outliers) simply counts the number of points that are not perfectly fitted, while the $\ell^\infty$-loss only cares about outliers, that is, it measures only the largest error on any data point. The parameter $p$ can be used to smoothly interpolate between these two extreme options in practice. Notably, concave loss functions ($p<1$) are explicitly used to obtain very outlier-robust methods \cite{wang2013semi,jiang2015robust,de2018concave}. See also \cite{janocha2016loss} for an analysis of the impact of different loss functions on neural network training.

Without loss of generality, we always assume that~$n \ge d+1$ because otherwise the problem can be solved in the lower-dimensional affine hull of the input data points.\footnote{To see this, assume that all~$\mathbf{x}_i$ are contained in an affine subspace $A$ with dimension strictly less than~$d$. Using a bijective affine transformation $T\colon A\to \R^{\dim A}$, we now can solve the equivalent  $k$-\textsc{ReLU}($\ell$) problem with the lower-dimensional data points $(T(x_i),y_i)$, $i\in[n]$. The solution weights for the original problem can then be obtained by composing $T$ with the affine transformation given by the weights of the modified problem.}

The core theme of our work is to better understand how the 
dimension parameter~$d$ influences the computational 
complexity of ReLU network training as defined above. 
To this end, we conduct a parameterized complexity analysis~\cite{DF13}.
Before moving on to study the key parameter~$d$,
let us briefly discuss other parameters occurring in
our setting.
The most natural other parameters are:
the number~$k$ of ReLU neurons in the hidden layer,
the number~$n$ of input points, and the maximum target error~$\gamma$.

It turns out that the parameterized 
complexity for these three parameters is already settled due to 
the known literature and simple observations.
First, the case $k=1$ (that is, $1$-\textsc{ReLU($\ell$)}) is known
to be NP-hard for~\mbox{$\ell^2$-loss}~\cite{DWX20,GKMR21} and we even extend the NP-hardness to $\ell^p$-loss for arbitrary~\mbox{$p\in[0,\infty[$} (see \Cref{thm:hard}); this renders the parameter~$k$
hopeless in terms of getting efficient parameterized algorithms.
For the parameter~$n$, fixed-parameter tractability was already observed by~\citeA{GKMR21} (see also related work).
Finally, for $\gamma=0$, the case~$k=1$ is polynomial-time solvable~\cite{GKMR21} and for~$k\ge 2$ NP-hardness is known~\cite{BJW19,GKMR21}.
Thus, the dimension~$d$ clearly is the most interesting parameter also
from a theoretical point of view. We
close some knowledge gaps from the literature with respect to~$d$ 
by proving strong computational hardness results as well as matching upper bounds.
Before discussing our results in more detail, let us first 
review the closely related literature.

\paragraph*{Related Work.}
The NP-hardness of empirical risk minimization with $\ell^2$-loss for a single ReLU was shown independently by \citeA{DWX20} and \citeA{GKMR21}.
The work of \citeA{GKMR21} is probably the one closest to our work.
They provided an in-depth study concerning tight hardness results 
for depth-2 ReLU networks such as NP-hardness, conditional running time lower bounds, and hardness of approximation.
\citeA{ABMM18} provided a polynomial-time algorithm for $k$-ReLU($\ell$)
for $d\in O(1)$ and convex loss~$\ell$; in terms of parameterized algorithmics,
this is an XP-algorithm for parameter~$d$: 
the degree of the polynomial of the running time 
(only) depends on~$d$. The underlying idea of their algorithm is to basically iterate over all
$O(n^d)$ hyperplane partitions of the $n$~data points.
Indeed, as pointed out by \citeA{GKMR21}, 
since there are at most~$2^n$ partitions, 
the same algorithm implies fixed-parameter tractability 
for the parameter~$n$. Moreover, \citeA{GKMR21} remarked that the 
well-known Exponential Time Hypothesis (ETH) implies that no 
\mbox{$2^{o(n)}$-time} algorithm exists.
Deciding whether zero error ($\gamma =0$) is possible 
(that is, \emph{realizable} data) is 
polynomial-time solvable for a single ReLU by linear programming \cite{GKMR21}
and NP-hard for two \mbox{ReLUs}~\cite{GKMR21}.
Approximation has been subject to 
further works \cite{DWX20,DGKK20,GKMR21}.
Furthermore,
\citeA{BJW19} and \citeA{CKM21} showed fixed-parameter tractability 
results for related but different learning concepts of ReLU networks and \citeA{BDL20} studied the computational complexity of ReLU networks where the output neuron is also a ReLU.
\citeA{pilanci2020neural} show that training a 2-layer neural network can be reformulated as a convex program. However, the implications on the computational complexity are limited since their result requires the number of hidden neurons to be very large.
\citeA{bertschinger2022training} show that training 2-layer neural networks is complete for the complexity class $\exists\R$ (existential theory of the reals), implying that the problem is presumably not contained in NP. They generalize a previous result by \citeA{AKM21}, who showed the same fact for specifically designed, more complex architectures.

Finally, note that the number of dimensions appears naturally in parameterized 
complexity studies
for geometric problems~\cite{GKR09,KKW15}; moreover, it
occurs also in recent studies for
principal component analysis (PCA)~\cite{FGS20,SFGP19}
and in computer vision~\cite{CCN20}.

\paragraph*{Our Contributions.}
Essentially focusing on the influence of the dimension parameter~$d$
(which so far has been neglected in the literature), we 
provide two main contributions in terms of worst-case complexity analysis:
First, training a two-layer ReLU neural network is already computationally intractable 
even for a single hidden neuron and small~$d$ (\Cref{thm:hard}), that is, we show W[1]-hardness for
parameter~$d$ and provide an ETH-based lower bound of~$n^{\Omega(d)}$ 
matching the running time upper bound of $n^{O(d)}$ of the brute-force algorithm 
due to \citeA{ABMM18}.
Hence, our result shows that the combinatorial search among all~$O(n^d)$ possible hyperplane partitions is essentially the best one can do.
Notably, our hardness results even hold for very sparse data points with binary labels.
It is particularly surprising that (parameterized) hardness already appears in the case of a single ReLU neuron because this model is almost a linear model. Linear models (like support vector machines) can be easily trained in polynomial time. Hence, the presence of a single nonlinearity, in the form of a single hyperplane of breakpoints, increases the computational complexity a lot. This also indicates that learning more complicated network architectures is expected to be even more difficult because the set of representable piecewise linear functions becomes much more complex \cite{hertrich2021towards}.

Second, on the positive side, for any $k\geq 1$, we extend the XP-result for convex loss functions by \citeA{ABMM18} to concave
loss functions (\Cref{thm:XP}). Note that for W[1]-hard problems, an XP-algorithm is the best one can hope for. Hence, we completely settle the computational complexity parameterized by dimension $d$ of training a two-layer ReLU neural network for any $\ell^p$-loss with $p\in[0,\infty[$.

Besides these two main findings filling gaps in the literature,
we also contribute a polynomial-time 
algorithm (for arbitrary dimension) for training a single-neuron
ReLU network when using the \mbox{$\ell^\infty$-loss} function (\Cref{prop:poly}).
This generalizes the polynomial-time result due to \citeA{GKMR21} for the zero-error case.
\Cref{Tab:Overview} provides an overview of our results for the special case of a single hidden neuron ($k=1$).

\begin{table*}[t]
	\centering
	\caption{(Parameterized) computational complexity of training a single ReLU neuron with respect to parameter $d$ (input dimension) for $\ell^p$-loss functions.}
	\begin{tabular}{lll}
		\toprule
		& Hardness & Algorithm \\
		\midrule
		$p\in\left[0,1\right[$ & W[1]-h. + no $n^{o(d)}$-time alg. (\Cref{thm:hard}) & $n^{O(d)}\poly(n,d)$ (\Cref{thm:XP})\\
		$p\in\left[1,\infty\right[$ & W[1]-h. + no $n^{o(d)}$-time alg. (\Cref{thm:hard}) & $n^d\poly(n,d)$ \cite{ABMM18}\\
		$p=\infty$ & - & $\poly(n,d)$ (\Cref{prop:poly})\\
		\bottomrule
	\end{tabular}%
	\label{Tab:Overview}
\end{table*}

While hardness of approximation was already shown by \citeA{GKMR21}, we complement this by parameterized hardness of exact solutions.
Thus, our work points the way ahead towards a need of combining both algorithmic approaches and to look for approximation algorithms in FPT time.

\paragraph*{Parameterized Complexity.}
We assume the reader to be familiar with basic concepts of computational complexity theory.
Parameterized complexity is a multivariate approach to study the time complexity of computational problems~\cite{DF13,Cyg15}.
An instance~$(x,k)$ of a parameterized problem~$L\subseteq\Sigma^*\times\N$ consists of
a classical problem instance~$x\in\Sigma^*$ and a \emph{parameter} value~$k\in\N$.
A parameterized problem is \emph{fixed-parameter tractable (fpt)} (contained in the class FPT) if there exists an algorithm
solving any instance~$(x,k)$ in $f(k)\cdot|x|^{O(1)}$ time, where~$f$ is a function solely depending on~$k$.
Note that fixed-parameter tractability implies polynomial time for constant parameter values where, importantly, the degree of the polynomial is independent from the parameter value.
The class W[1] contains parameterized problems which are presumably not in FPT (e.g.~\textsc{Clique} parameterized by the size of the requested clique).
Parameterized intractability can be shown in terms of W[1]-hardness which is defined via \emph{parameterized reductions}.
A parameterized reduction from~$L$ to~$L'$ is an algorithm that maps an instance~$(x,k)$ in $f(k)\cdot |x|^{O(1)}$ time to an instance~$(x',k')$
such that~$k'\le g(k)$ for some function~$g$ and $(x,k)\in L$ if and only if~$(x',k')\in L'$.
The class XP contains all parameterized problems which can be solved in polynomial time if the parameter is a constant, that is,
in time~$f(k)\cdot |x|^{g(k)}$.
It is known that $\text{FPT}\subseteq \text{W[1]}\subseteq \text{XP}$ and that~$\text{FPT}\subsetneq\text{XP}$. 

\paragraph*{Exponential Time Hypothesis.}
The Exponential Time Hypothesis (ETH)~\cite{IP01} states that \textsc{3-SAT} cannot be solved in subexponential time in the number~$n$ of variables of the Boolean input formula.
That is, there exists a constant~$c > 0$ such that \textsc{3-SAT} cannot be solved in~$O(2^{cn})$ time.
The ETH implies that $\text{FPT}\neq\text{W[1]}$ (and hence $\text{P}\neq\text{NP}$)~\cite{Cyg15}.
It also implies running time lower bounds, for example, that \textsc{Clique} cannot be solved in
$\rho(k)\cdot n^{o(k)}$ time for any function~$\rho$, where~$k$ is the size of the sought clique and $n$~is the number of graph vertices~\cite{Cyg15}.

\paragraph*{Notation.} For $n\in\N$, let~$[n]:=\{1,\ldots,n\}$.

\section{Hardness of Training a Single ReLU with $\ell^p$-Loss in Small Dimension}

In this section, we show that there is no hope to obtain an FPT algorithm with respect to parameter $d$ for training even the simplest possible architecture consisting of a single neuron with respect to the $\ell^p$-loss for any~$p\in\left[0,\infty\right[$. To this end, we show intractability of the problem \mbox{1-\textsc{ReLU}($\ell^p$)}.
For $p=0$, the problem is to minimize the number of data points
that are not perfectly fitted.

\begin{theorem}\label{thm:hard}
	For $p\in[0,\infty[$, 
	$1$-\textsc{ReLU($\ell^p$)} is NP-hard, W[1]-hard with respect to dimension~$d$, and it cannot be solved in $\rho(d)\cdot n^{o(d)}$ time for any computable
	function~$\rho$ unless the Exponential Time Hypothesis fails.
\end{theorem}

Note that \citeA{GKMR21} proved NP-hardness and conditional running time lower bounds for additive approximation of~$k$-\textsc{ReLU($\ell^2$)} for~$k\ge 1$.
Their running time lower bound for~$k=1$ is based on a newly introduced assumption of inapproximability of finding dense subgraphs and the lower bound for~$k>1$ assumes the \emph{Gap Exponential Time Hypothesis}~\cite{Dinur16} (which implies the ETH).
Their results are achieved via gap reductions from the problems of finding dense subgraphs and coloring hypergraphs. The reductions are focused on providing a gap which guarantees the approximation hardness. They achieve this by using a ``large'' number~$d$ of dimensions (typically equal to the number of vertices of the input (hyper)graph).
For our purpose, however, we need a more fine-grained parameterized reduction where~$d$ is ``small''.
To this end, we reduce from a colored variant of \textsc{Clique} such that~$d$ is linearly upper-bounded in the size of the clique.

\begin{proof}[Proof of \Cref{thm:hard}]
	We reduce from the following problem.
	
	\problemdef{Multicolored Clique}
	{An undirected graph~$G=(V,E)$ where the vertices are colored with~$k$ colors.}
	{Does~$G$ contain a $k$-clique (a complete subgraph with~$k$ vertices) with exactly one vertex from each color?}
	
	\textsc{Multicolored Clique} is NP-hard, W[1]-hard with respect to~$k$, and not solvable in \mbox{$\rho(k)\cdot |V|^{o(k)}$} time for any computable function~$\rho$ unless the Exponential Time Hypothesis fails~\cite{Cyg15}. We give a parameterized reduction (which is also polynomial-time) from \textsc{Multicolored Clique} to $1$-\textsc{ReLU($\ell^p$)} where the dimension of the data points
	is $d=2k$. Hence, the theorem follows.
	
	Let $G=(V,E)$ be an undirected graph with $N\coloneqq \abs{V}$ vertices and let $c\colon V\to[k]$ be a vertex coloring.
	We denote by $V_j=\{v_{j,1},\dots,v_{j,N_j}\}$ the set of vertices with color $j$, where $N_j\coloneqq \abs{V_j}$.
	In the following, we construct a set of data points from $\R^{2k}$ with labels in $\{0,1\}$, as well as a target error $\gamma\in \R$, such that these data points can be fitted by a ReLU function with $\ell^p$-error at most~$\gamma$ if and only if a multicolored $k$-clique exists in $G$.
	
	We set the target error to $\gamma\coloneqq N-k$.
	Next, we define a small value $0 < \delta < 1$ such that making an absolute error of value $1-\delta$ on $N-k+1$ different input points already exceeds the threshold~$\gamma$.
	For $p=0$, we simply choose $\delta\coloneqq0.5$.
	For~$p > 0$, let $\tilde{p}\coloneqq \max\{p,1\}$ and $\delta\coloneqq 1/(2\tilde{p}(N-k+1))$.
	This yields
	\begin{align}
		(1-\delta)^p(N-k+1)
		&\geq (1-\tilde{p}\delta)(N-k+1)\notag\\
		&>(1-2\tilde{p}\delta)(N-k+1)\notag\\
		&= N-k=\gamma,\label{Equ:delta}
	\end{align}
	where, in the case $p>1$, the first inequality follows from Bernoulli's inequality, and in the case $p\leq1$, it follows from $x^p\geq x$ for all~$x\in[0,1]$.
	
	In addition, we define a large integer~$M$ such that making an absolute error of~$\delta$ on~$M$ different input points also exceeds the threshold $\gamma$.
	For $p=0$, we choose $M\coloneqq N-k+1$. For~$p>0$, we set $M\coloneqq \lceil\gamma/\delta^p\rceil+1$, which implies
	\begin{equation}\label{Equ:M}
		M\delta^p>\gamma.
	\end{equation}
	Note that $\gamma\in O(N)$ and $1/\delta\in O(N)$. Thus, $M$ is polynomially bounded in the size of~$G$.
	
	\begin{figure}[t]
		\centering
		\begin{tikzpicture}[scale=2]
			\tikzset{vertex/.style={draw, fill=white, circle, inner sep=0pt,minimum size=5pt}}
			\tikzset{blocker/.style={draw, fill, circle, inner sep=0pt,minimum size=7pt}}
			\fill[shade] (-1.5,1.309/2) rectangle (1.5,1.1);
			\draw[black!20!white] (0,0) circle (1cm);
			\node[vertex,label=above:$\tilde{\mathbf{x}}_1$] (1) at (0,1) {};
			\node[vertex,label=right:$\tilde{\mathbf{x}}_2$] (2) at (0.951,0.309) {};
			\node[vertex,label=right:$\tilde{\mathbf{x}}_3$] (3) at (0.588,-0.809) {};
			\node[vertex,label=left:$\tilde{\mathbf{x}}_4$] (4) at (-0.588,-0.809) {};
			\node[vertex,label=left:$\tilde{\mathbf{x}}_5$] (5) at (-0.951,0.309) {};
			\draw[black!20!white] (1) -- (2) node[black,midway,blocker] {};
			\draw[black!20!white] (1) -- (3) node[black,midway,blocker] {};
			\draw[black!20!white] (1) -- (4) node[black,midway,blocker] {};
			\draw[black!20!white] (1) -- (5) node[black,midway,blocker] {};
			\draw[black!20!white] (2) -- (3) node[black,midway,blocker] {};
			\draw[black!20!white] (2) -- (4) node[black,midway,blocker] {};
			\draw[black!20!white] (2) -- (5) node[black,midway,blocker] {};
			\draw[black!20!white] (3) -- (4) node[black,midway,blocker] {};
			\draw[black!20!white] (3) -- (5) node[black,midway,blocker] {};
			\draw[black!20!white] (4) -- (5) node[black,midway,blocker] {};
			\draw[dashed,thick] (-1.5,1.309/2) -- (1.5,1.309/2);
		\end{tikzpicture}
		\caption{Schematic illustration of the reduction from \textsc{Multicolored Clique}. Shown are two dimensions corresponding to one of the~$k$ colors. The white points~$\tilde{\mathbf{x}}_1, \ldots, \tilde{\mathbf{x}}_5$ correspond to vertices of that color and have label~1. Black points indicate~$M$ copies of the corresponding middle point with label~0.
			The dashed line indicates the hyperplane defined by the weight vector~$\mathbf{w}$ of the ReLU neuron and the shaded area indicates the predictions of the neuron (darker means larger values).
			The idea is that exactly one white point can be predicted correctly (which selects the corresponding vertex to be in the clique) without predicting a black point incorrectly and thereby exceeding the error.}
		\label{fig:construction}
	\end{figure}
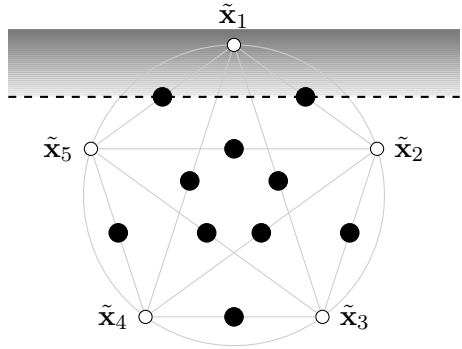
	
	Let $N_{\max}\coloneqq\max_{j\in[k]} N_j$ be the maximum number of vertices of one color. For our reduction we need $N_{\max}$ distinct rational points on the unit circle centered at the origin. For example, one can choose
	\[\tilde{\mathbf{x}}_i\coloneqq \left(\frac{1-i^2}{1+i^2},\frac{2i}{1+i^2}\right)\in\Q^2\]
	for each $i=1,2,\dots,N_{\max}$~\cite{ST94}.
	For each vertex $v_{j,i}\in V$, $j\in[k]$, $i\in[N_j]$, let \mbox{$\mathbf{x}_{j,i}=(\mathbf{0}_{2j-2}, \tilde{\mathbf{x}}_i, \mathbf{0}_{2k-2j})\in\R^{2k}$} be the point $\tilde{\mathbf x}_i$ lifted to $2k$ dimensions, where each color corresponds to two dimensions. Here, we use the notation~$\mathbf{0}_r$ for the $r$-dimensional zero-vector. We add two types of data points to our instance (see \Cref{fig:construction} for an example). First, for each vertex $v_{j,i}\in V$, add the point $(\mathbf x_{j,i},1)\in\R^{2k}\times\R$. Second, for each pair of distinct vertices $v_{j,i} \neq v_{j',i'}\in V$, if they cannot both be part of a multicolored clique because they are non-adjacent or have the same color, then add $M$ copies of the point $\big((\mathbf x_{j,i} + \mathbf x_{j',i'})/2,0\big)\in\R^{2k}\times\R$.
	This finishes the construction.
	
	We now show that there is a multicolored clique of size $k$ in $G$ if and only if these data points can be fitted by a ReLU with $\ell^p$-error at most $\gamma$. This then completes our reduction from \textsc{Multicolored Clique} to $1$-\textsc{ReLU($\ell^p$)} and hence implies the theorem.
	
	For the first direction, assume that the vertices~$v_{1,i_1},\dots,v_{k,i_k}$ form a multicolored clique of size $k$ in $G$.
	We define \mbox{$\varepsilon\coloneqq 1 - \max_{i\neq i'\in[N_{\max}]} \langle \tilde{\mathbf{x}}_i, \tilde{\mathbf{x}}_{i'}\rangle$}. Observe that $\varepsilon>0$, since the points $\tilde{\mathbf{x}}_i$, $i\in[N_{\max}]$, are distinct points on the unit circle.
	Let \mbox{$\mathbf{w}\coloneqq 2/\varepsilon \cdot (\tilde{\mathbf{x}}_{i_1}, \tilde{\mathbf{x}}_{i_2}, \dots, \tilde{\mathbf{x}}_{i_k})\in\R^{2k}$} and \mbox{$b\coloneqq 1-2/\varepsilon$}.
	We claim that the ReLU function \mbox{$f(\mathbf x)=[\langle\mathbf{w},\mathbf{x}\rangle + b]_+$} achieves an $\ell^p$-error of exactly~\mbox{$\gamma= N-k$}.
	To see this, first note that for each $j\in[k]$, we have \[\langle\mathbf{w},\mathbf{x}_{j,i_j}\rangle + b = 2/\varepsilon\cdot\langle\tilde{\mathbf{x}}_{i_j},\tilde{\mathbf{x}}_{i_j}\rangle + 1-2/\varepsilon = 1,\]
	where we used that $\tilde{\mathbf{x}}_{i_j}$ lies on the unit circle.
	Hence, the $k$ points~$\mathbf{x}_{1,i_1},\ldots,\mathbf{x}_{k,i_k}$ are perfectly fitted.
	Second, for each~$v_{j,i}\in V\setminus\{v_{1,i_1},\dots,v_{k,i_k}\}$, we have
	\begin{align*}\langle\mathbf{w},\mathbf{x}_{j,i}\rangle + b = 2/\varepsilon\cdot\langle\tilde{\mathbf{x}}_{i_j},\tilde{\mathbf{x}}_{i}\rangle + 1-2/\varepsilon \leq 2/\varepsilon\cdot(1-\varepsilon) + 1-2/\varepsilon = -1,\end{align*}
	where the inequality follows from our choice of $\varepsilon$.
	Hence, for each of these $N-k$ points, we have $f(\mathbf{x}_{j,i})=0$, that is, we incur an error of~$1$. Finally, for each pair of distinct vertices $v_{j,i} \neq v_{j',i'}\in V$ that are either non-adjacent or have the same color, note that at most one of the two vertices can belong to the clique. Thus, making use of our two calculations above, we obtain
	\begin{align*}
	\langle \mathbf w, (\mathbf{x}_{j,i} + \mathbf{x}_{j',i'})/2 \rangle + b = ((\langle \mathbf w, \mathbf{x}_{j,i} \rangle + b) + (\langle \mathbf w, \mathbf{x}_{j',i'} \rangle + b))/2 \leq (1 -1)/2 = 0.
	\end{align*}
	Hence, all points with label~$0$ are fitted exactly and the total $\ell^p$-error is equal to $\gamma= N-k$.
	
	For the reverse direction, suppose that there exist~\mbox{$\mathbf{w}\in\R^{2k}$} and~$b\in\R$ such that the ReLU function \mbox{$f(\mathbf x)=[\langle\mathbf{w},\mathbf{x}\rangle + b]_+$} achieves an $\ell^p$-error of at most \mbox{$\gamma= N-k$}. We show that the set \[C\coloneqq\{v_{j,i}\in V\mid f(\mathbf{x}_{j,i})>\delta\}\] forms a multicolored clique in $G$.
	First, observe that \mbox{$\abs{C}\geq k$}, because otherwise all data points associated with vertices in $V\setminus C$ would incur a total $\ell^p$-error of at least \mbox{$(1-\delta)^p(N-k+1)$}, which is larger than $\gamma$ by \eqref{Equ:delta}.
	Hence, it remains to show for each pair of vertices \mbox{$v_{j,i}\neq v_{j',i'}\in C$} that they belong to different color classes and are adjacent. Suppose the contrary.
	Then, by construction, the $1$-\textsc{ReLU($\ell^p$)} instance also contains $M$ copies of the point $\big((\mathbf x_{j,i} + \mathbf x_{j',i'})/2,0\big)\in\R^{2k}\times\R$. From \mbox{$\langle\mathbf{w},\mathbf{x}_{j,i}\rangle+b>\delta$} and $\langle\mathbf{w},\mathbf{x}_{j',i'}\rangle+b>\delta$, it follows by linearity that \[f((\mathbf x_{j,i} + \mathbf x_{j',i'})/2)\geq \langle\mathbf{w},(\mathbf x_{j,i} + \mathbf x_{j',i'})/2\rangle+b >\delta.\] Thus, we incur an $\ell^p$-error of at least $M\delta^p$, which is larger than $\gamma$ by~\eqref{Equ:M}.
	Hence, $C$ is indeed a multicolored $k$-clique.
\end{proof}
A closer inspection of the above proof reveals that hardness even holds for a more restricted problem.

\begin{corollary}\label[corollary]{Cor:hard}
	For $p\in [0,\infty[$, $1$-\textsc{ReLU($\ell^p$)} is NP-hard, W[1]-hard with respect to~$d$, and cannot be solved in \mbox{$\rho(d)\cdot n^{o(d)}$} time for any computable function~$\rho$ (assuming the Exponential Time Hypothesis) even if all input data points contain at most four non-zero entries and have binary labels.
\end{corollary}

We further remark that the basic idea of the reduction in the proof of \Cref{thm:hard} also works for more general loss functions.
Essentially, the only necessary condition is that the value $M$ can be chosen such that it is polynomially bounded in the size of the graph $G$ and satisfies an inequality analogous to~\eqref{Equ:M} where~$(\cdot)^p$ is replaced by the corresponding loss function.
We refrain from giving a precise formalization here.
Moreover, it is natural to expect that $k$-ReLU with $k>1$ is computationally even more difficult than the one-neuron case. Hence, our hardness results should also hold there as well. 
Indeed, we expect this to be also true for deeper neural networks; however, a formal proof would require a more profound understanding of the complicated functions expressible with deeper networks and is left for future work.

Our findings tell us that in order to achieve fixed-parameter tractability, one has to consider other parameters to combine with the dimension~$d$.
A natural parameter is the target loss~$\gamma$.
However, this is not a promising parameter since it can be made arbitrarily small by scaling all values.
If we consider the number~$\sigma$ of different coordinate values of the~$\mathbf{x}_i$, then we trivially obtain fixed-parameter tractability in combination with~$d$ since the overall number of different data points is at most~$\sigma^d$.
Hence, the algorithm by~\citeA{ABMM18} runs in~$\sigma^{d^2}\cdot \poly(nd)$ time.

To sum up, identifying promising parameters (or parameter combinations) 
to obtain tractable cases remains 
a challenge worthwhile further investigation.

\section{Polynomial-time Algorithm for a Single ReLU with Maximum Norm}\label{Sec:Poly}

As pointed out by \citeA{GKMR21}, deciding whether given data points are realizable by a single ReLU neuron (that is, whether $\gamma =0$) can be done in polynomial time via linear programming. In other words, it is possible to check whether the input points can be perfectly fitted by a single ReLU neuron and, in case of a positive answer, to find the corresponding weights in polynomial time. Recall that the same problem is NP-hard in the case of two (or more) neurons by \citeA{GKMR21}.

In this section, we extend this result to minimizing the $\ell^\infty$-loss, that is, minimizing the maximum prediction error.
In fact, we provide a polynomial-time optimization algorithm (not only decision) for a problem variant that generalizes $\ell^\infty$-loss minimization.
In this variant, the real labels~$y_i$ for the data points $\mathbf x_i$ are replaced by target intervals $[\alpha_i,\beta_i]$ with $\alpha_i\leq\beta_i$ and we aim to minimize the maximum deviation of a prediction from its corresponding target interval.
To this end, we define $\dist_{\alpha,\beta}(t)\coloneqq\max\{\alpha-t,0,t-\beta\}$ to be the \emph{distance} of $t\in\R$ to the interval $[\alpha,\beta]$.

\optproblemdef{ReLU($\ell^\infty$-Interval)}
{Data points~$\mathbf{x}_1,\ldots,\mathbf{x}_n\in \R^d$ and interval boundaries $\alpha_1\le\beta_1,\ldots,\alpha_n\le\beta_n\in\R$.}
{Find~$\mathbf{w}\in\R^d$ and~$b\in\R$ that minimize $\max_{i\in[n]}\ \dist_{\alpha_i,\beta_i}([\langle\mathbf{w},\mathbf{x}_i\rangle + b]_+)$.}
Note that we obtain $\ell^\infty$-loss minimization by setting \mbox{$\alpha_i=\beta_i=y_i$} for all~$i\in[n]$.

\begin{proposition}\label[proposition]{prop:poly}
	\textsc{ReLU($\ell^\infty$-Interval)} can be solved in polynomial time.
\end{proposition}
\begin{proof}
	We show that an optimal solution can be found via solving a series of linear programs.
	For each $i\in[n]$ with $\alpha_i>0$, our algorithm finds out whether the optimal objective value~$\gamma^*$ is larger or smaller than $\alpha_i$.
	In the first case, the prediction~\mbox{$\langle\mathbf{w},\mathbf{x}_i\rangle + b$} is allowed to be arbitrarily small, while in the second case we need to ensure the lower bound \mbox{$\langle\mathbf{w},\mathbf{x}_i\rangle + b \geq \alpha_i - \gamma^*$}.
	Therefore, we implement a binary search to find an interval in which~$\gamma^*$ is contained as follows.
	Let~$\{\tilde{\alpha}_1,\tilde{\alpha}_2,\dots,\tilde{\alpha}_r\}$ be the set of all distinct positive~$\alpha_i$-values, $i\in[n]$, sorted by index such that \mbox{$0\eqqcolon\tilde{\alpha}_0<\tilde{\alpha}_1<\dots<\tilde{\alpha}_r<\tilde{\alpha}_{r+1}\coloneqq \infty$}.
	Let~\mbox{$s^*\in[r+1]$} denote the (unknown) index with \mbox{$\gamma^*\in\left[\tilde{\alpha}_{s^*-1}, \tilde{\alpha}_{s^*}\right[$}.
	For each $s\in[r+1]$, we define a linear program denoted by LP($s$) which minimizes the maximum deviation under the assumption that only the predictions for data points~$\mathbf x_i$ with $\alpha_i\geq\tilde{\alpha}_s$ are bounded from below, while all other predictions can be arbitrarily small.	
	\begin{equation}\tag{LP($s$)}
		\begin{aligned}
			\min_{\mathbf w, b, \gamma} \quad&\gamma\\
			\text{s.t.}\quad & \langle\mathbf{w},\mathbf{x}_i\rangle + b \in [\alpha_i - \gamma, \beta_i + \gamma], &\quad i\in[n] \text{ with } \alpha_i \geq \tilde{\alpha}_{s},\\
			& \langle\mathbf{w},\mathbf{x}_i\rangle + b \leq \beta_i + \gamma, &\quad i\in[n] \text{ with } \alpha_i < \tilde{\alpha}_{s},\\
			& \gamma \geq -\beta_i, &\quad i\in[n],\\
			& \gamma \geq 0.
		\end{aligned}
	\end{equation}
	Here, the constraint $\gamma \geq -\beta_i$ is only relevant if $\beta_i<0$. In this case, it is needed to ensure that the error is at least $-\beta_i$ because a ReLU unit can only output nonnegative values.
	
	Suppose we already knew the optimal index $s^*$. Observe that, by construction of LP($s^*$), a triplet $(\mathbf{w}, b, \gamma)$ is an optimal solution for LP($s^*$) if and only if $(\mathbf{w}, b)$ is optimal for the problem \mbox{\textsc{ReLU($\ell^\infty$-Interval)}} with objective value $\gamma$. Hence, it only remains to show how~$s^*$ can be found.
	To this end, let $\gamma(s)$ be the objective value of~LP($s$) for each $s\in[r+1]$. Note that $\gamma(s_1)\geq\gamma(s_2)$ for $s_1<s_2$ because the set of constraints of~LP($s_1$) is a superset of the constraints of LP($s_2$). Hence, for $s>s^*$, it follows that \[\gamma(s)\leq\gamma(s^*)=\gamma^*<\tilde{\alpha}_{s^*}\leq \tilde{\alpha}_{s-1}.\] Similarly, for $s<s^*$, we obtain \[\gamma(s)\geq\gamma(s^*)=\gamma^*\geq \tilde{\alpha}_{s^*-1}\geq \tilde{\alpha}_s.\] As a consequence, we can determine whether $s<s^*$, $s=s^*$, or $s\geq s^*$ by solving LP($s$) and comparing $\gamma(s)$ with $\tilde{\alpha}_s$ and $\tilde{\alpha}_{s-1}$.
	Thus, using binary search and solving $\mathcal{O}(\log n)$ linear programs, we can determine $s^*$ and the optimal solution $\gamma^*$ together with the corresponding weights~$\mathbf{w}$ and~$b$.
\end{proof}

We remark that, analogously to the original problem with labels $y_i$, the zero-error case for the variant with intervals $[\alpha_i,\beta_i]$ can be solved with a single linear program instead of a binary search: It suffices to run LP($1$) once. This results in objective value $0$ if and only if all data points can be fitted precisely within their intervals.

\section{Polynomial-time Algorithm for Concave Loss in Fixed Dimension}

In this section, we prove that, for any loss function of the form $\ell(\hat{y},y)=\tilde{\ell}(\abs{\hat{y}-y})$ where $\tilde{\ell}\colon\R_{\geq0}\to\R_{\geq0}$ is concave, the problem $k$-\textsc{ReLU}($\ell$) is polynomial-time solvable for constant~$d$ (that is, it is in XP with respect to~$d$). In particular, this covers the case of $\ell^p$-loss for $p\in[0,1[$.
Notably, concave loss functions can yield increased robustness by mitigating the influence of outliers.
For convex loss functions, in particular for the $\ell^p$-loss with $p\geq1$, an analogous result has already been shown by \citeA[Theorem~4.1]{ABMM18}. More precisely, they showed that, if $\ell$ is convex, then $k$-\textsc{ReLU}($\ell$) can be solved in $O(2^k n^{dk} \poly(n,d,k))$ time.
The idea of their algorithm is essentially to try out all $O(n^d)$ hyperplane partitions of the~$n$ input points for each of the $k$ ReLU neurons
and solve a corresponding convex program.

For the concave case, we follow a similar approach. The only but decisive difference is that the occurring subproblems are not convex programs. Instead, we show that they can be written as optimization problems over polyhedra with an objective function that is piecewise concave. It is well-known that global optima of concave problems always occur at a vertex of the feasible polyhedron~\cite{Benson95} and that it is possible to enumerate all vertices of the polyhedron in XP-time~\cite{KP03}. However, since in our case the objective function is only piecewise concave, it is possible that no vertex is a global optimum. Instead, we need to enumerate all vertices of all concave pieces of the feasible region. We show that this can still be done in XP-time, completing the parameterized complexity classification picture.

\begin{theorem}\label{thm:XP}
	For every loss function being of the form $\ell(\hat{y},y)=\tilde{\ell}(\abs{\hat{y}-y})$ with a concave function~\mbox{$\tilde{\ell}\colon\R_{\geq0}\to\R_{\geq0}$}, the problem $k$-\textsc{ReLU}($\ell$) is solvable in time $2^k(nk)^{O(dk)}\poly(n,d,k)$.
\end{theorem}

\begin{proof}
	Following the approach by \citeA[Algorithm~1]{ABMM18}, for each neuron~\mbox{$j\in[k]$},
	we consider each coefficient~$a_j\in\{-1,1\}$ and each hyperplane partition $P_+^j \cup P_-^j = [n]$, $P_+^j \cap P_-^j = \emptyset$, of the $n$ (indices of the) data points (that is, there exists a $(d-1)$-dimensional hyperplane, defined by a vector~$\mathbf{w}_j$ and a bias~$b_j$, separating $P_+^j$ and~$P_-^j$, compare \Cref{Fig:HyperplanePartition}).
	Here, $P_+^j$ is the \emph{active} set, where $\langle\mathbf{w}_j,\mathbf{x}_i\rangle+b_j \ge 0$ shall hold for each~$i\in P_+^j$ and \mbox{$\langle\mathbf{w}_j,\mathbf{x}_i\rangle+b_j \leq 0$} for each~$i\in P_-^j$.
	As in the algorithm by \citeA{ABMM18}, this results in a total of (at most) $2^kn^{dk}$ subproblems.
	For fixed coefficients $a_j$ and fixed partitions $(P_+^j, P_-^j)$, $j\in[k]$, the corresponding subproblem (compare Line~8 in Algorithm~1 of \citeA{ABMM18}) is the following:
	\begin{equation}\label{Equ:sub}
		\begin{aligned}
			\min_{\mathbf{w}_j, b_j} \quad&\sum_{i=1}^n \tilde{\ell}\left(\abs{y_i-\sum_{j\colon i\in P_+^j} a_j(\langle\mathbf{w}_j,\mathbf{x}_i\rangle + b_j) }\right)\\
			\text{s.t.}\quad & \langle\mathbf{w}_j, \mathbf{x}_i\rangle + b_j \leq 0, \qquad j\in[k], i\in P_-^j,\\
			& \langle\mathbf{w}_j, \mathbf{x}_i\rangle + b_j \geq 0, \qquad j\in[k], i\in P_+^j.
		\end{aligned}
	\end{equation}
	In the following, we show that this problem can be solved in XP-time with respect to~$d$.
	
	As argued in the introduction, we may assume without loss of generality that the affine hull of the data points $\mathbf{x}_i$, $i\in[n]$, is the whole space~$\R^d$ because, otherwise, we could solve the problem within a lower-dimensional affine subspace.
	We first show that this implies that the feasible region $P\subseteq \R^{kd+k}$ of \eqref{Equ:sub} is \emph{pointed}, that is, it has at least one vertex. More precisely, we show that the zero vector $\mathbf{0}_{kd+k}$ is a vertex of $P$. To do so, we need to show that it satisfies $kd+k$ linearly independent constraints of~\eqref{Equ:sub} with equality. Since $\mathbf{0}_{kd+k}$ satisfies every constraint of~\eqref{Equ:sub} with equality, we only need to show that there exist $kd+k$ linearly independent rows. We write $\mathbf r_{ij}\coloneqq(\mathbf{0}_{d(j-1)},\mathbf x_i,\mathbf{0}_{d(k-j)}, \mathbf e_j)\in \R^{kd+k}$, $i\in[n]$, $j\in[k]$, for the~$kn$ rows of the constraint matrix, where $\mathbf e_j\in\{0,1\}^k$ is the $j$-th unit vector. By our assumption that the affine hull of the data points is the whole space $\R^d$, there exists a subset $S\subseteq[n]$ of $d+1$ indices such that the $d+1$ vectors $\mathbf x_i$, $i\in S$, are affinely independent. This implies that, for each fixed $j$, the $d+1$ rows $\mathbf{r}_{ij}$ are linearly independent. Moreover, since, for $j_1\neq j_2$ and arbitrary $i_1,i_2\in[n]$, two rows $\mathbf r_{i_1j_1}$ and $\mathbf r_{i_2j_2}$ have non-zero entries only in distinct columns, it follows that the $kd+k$ rows $\mathbf r_{ij}$, $i\in S$, $j\in[k]$, are linearly independent. Hence, $P$ is pointed.
	
	Next, we divide the feasible region $P$ of \eqref{Equ:sub} into several polyhedral pieces, depending on the sign of the prediction error at each data point. Let $\mathbf{s} = (s_i)_{i\in[n]} \in\{-1,1\}^n$ be a sign vector and let
	\begin{align*}
	P(\mathbf{s})\coloneqq\Bigg\{
	(\mathbf w_1, \dots, \mathbf w_k, b_1,\dots, b_k)\in P 
	\mathrel{}\Bigg|\mathrel{}\forall i\in[n]\colon s_i\Bigg(y_i-\sum_{j\colon i\in P_+^j} a_j(\langle\mathbf{w}_j,\mathbf{x}_i\rangle + b_j)\Bigg)\geq0	\Bigg\}
	\end{align*}
	be the subset of the feasible region $P$ for which the sign of the prediction error for each data point~$\mathbf{x}_i$ coincides with $s_i$. Since $P$ is pointed, $P(\mathbf{s})$ must be pointed as well. Moreover, by definition, the prediction error of every data point has a fixed sign within $P(\mathbf{s})$, implying that the objective function of \eqref{Equ:sub} (as a sum of concave functions) is concave within $P(\mathbf{s})$ (compare \Cref{Fig:PiecewiseConcave}). In addition, the objective value is trivially bounded from below by $0$. Since the minimum of a bounded (from below), concave function over a pointed, nonempty polyhedral set is always attained by a vertex~\cite{Benson95}, it follows that $P(\mathbf{s})$ is either empty or must have a vertex minimizing the loss within $P(\mathbf{s})$. Since $P=\bigcup_{\mathbf{s}\in\{-1,1\}^n} P(\mathbf{s})$, it follows that the optimal solution of~\eqref{Equ:sub} must be a vertex of one of the polyhedral sets~$P(\mathbf{s})$. Hence, it suffices to enumerate all these vertices. Compare \Cref{Fig:PolyhedralSubdivision} for a schematic illustration of this idea. Each vertex of one of the polyhedra $P(\mathbf{s})$ is given by $kd+k$ linearly independent inequalities that hold with equality. For selecting these $kd+k$ equations, we have the choice between a total of~\mbox{$kn+n$} equations: the~$kn$ constraints of \eqref{Equ:sub} as well as the $n$ equations corresponding to the sign constraints defined by~$\mathbf{s}$. Note that these $n$ equations are the same for each $\mathbf{s}$ although the inequalities are different.
	
	\begin{figure}[t]
		\centering
		\begin{minipage}{.47\textwidth}
			\centering
			\begin{tikzpicture}[inner sep=3]
				\node[fill] at (-1ex,2ex) {};
				\node[fill] at (2ex,7ex) {};
				\node[fill] at (6ex,3ex) {};
				
				\node[fill] at (-5ex,-6ex) {};
				\node[fill] at (-11ex,-5ex) {};
				\node[fill] at (-10ex,1ex) {};
				
				\draw[very thick] (-6ex,10ex) -- (-1ex,-9ex);
				
				\node at (-12ex, 7ex) {$P_-^j$};
				\node at (5ex, -6ex) {$P_+^j$};
			\end{tikzpicture}
			\captionof{figure}{Two-dimensional illustration of the hyperplane partitions of the data points into $P_+^j=\{i\in[n]\mid\langle\mathbf{w}_j,\mathbf{x}_i\rangle+b_j\ge 0\}$ and $P_-^j=\{i\in[n]\mid \langle\mathbf{w}_j,\mathbf{x}_i\rangle+b_j < 0\}$.}
			\label{Fig:HyperplanePartition}
		\end{minipage}\hfill
		\begin{minipage}{.47\textwidth}
			\centering
			\begin{tikzpicture}[scale=1.7]
				\footnotesize
				\draw[-{Stealth}] (-1.2,0) -- (1.2,0) node[below] {prediction error};
				\draw[-{Stealth}] (0,0) node[below] {0} -- (0,1.2) node[right] {loss};
				\draw[domain=0:1,smooth,variable=\x, thick] plot (\x,{sqrt(\x)});
				\draw[domain=-1:0,smooth,variable=\x, thick] plot (\x,{sqrt(-\x)});
			\end{tikzpicture}
			\captionof{figure}{The contribution of a data point~$\mathbf{x}_i$ to the objective function is not globally concave. However, it is concave if the sign of the prediction error \mbox{$y_i-\textstyle\sum_{j\colon i\in P_+^j} a_j(\langle\mathbf{w}_j,\mathbf{x}_i\rangle + b_j)$} is fixed.}
			\label{Fig:PiecewiseConcave}
		\end{minipage}
	\end{figure}
	
	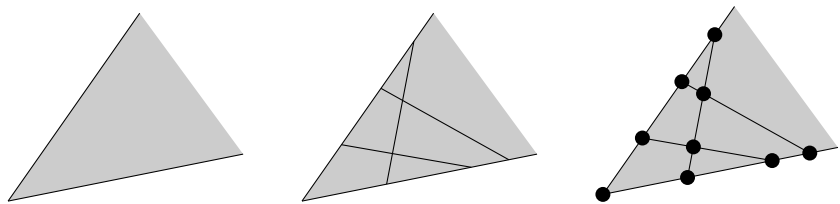
\begin{figure}[t]
		\centering
		\begin{tikzpicture}[scale=0.25, inner sep=2]
			
			\fill[black!20] (0,0) -- (12.5,2.5) coordinate (e1) -- (7,10) coordinate (e2) -- (0,0);
			
			\draw (0,0) -- (e1);
			\draw (0,0) -- (e2);
			
		\end{tikzpicture}\qquad
		\begin{tikzpicture}[scale=0.25, inner sep=2]
			
			\fill[black!20] (0,0) -- (12.5,2.5) coordinate (e1) -- (7,10) coordinate (e2) -- (0,0);
			
			\draw (0,0) -- (e1);
			\draw (0,0) -- (e2);
			
			\draw (2.1,3) coordinate (a1) -- (9,1.8) coordinate (a2);
			\draw (4.2,6) coordinate (b1) -- (11,2.2) coordinate (b2);
			\draw (5.95,8.5) coordinate (c1) -- (4.5,0.9) coordinate (c2);
			
		\end{tikzpicture}\qquad
		\begin{tikzpicture}[scale=0.25, inner sep=2]
			
			\fill[black!20] (0,0) -- (12.5,2.5) coordinate (e1) -- (7,10) coordinate (e2) -- (0,0);
			
			\draw (0,0) -- (e1);
			\draw (0,0) -- (e2);
			
			\draw (2.1,3) coordinate (a1) -- (9,1.8) coordinate (a2);
			\draw (4.2,6) coordinate (b1) -- (11,2.2) coordinate (b2);
			\draw (5.95,8.5) coordinate (c1) -- (4.5,0.9) coordinate (c2);
			
			\node[circle, fill] at (0,0) {};
			\node[circle, fill] at (a1) {};
			\node[circle, fill] at (a2) {};
			\node[circle, fill] at (b1) {};
			\node[circle, fill] at (b2) {};
			\node[circle, fill] at (c1) {};
			\node[circle, fill] at (c2) {};
			\node[circle, fill] at (intersection of a1--a2 and c1--c2) {};
			\node[circle, fill] at (intersection of b1--b2 and c1--c2) {};
			
		\end{tikzpicture}
		\caption{Schematic illustration of how an optimal solution to the subproblem \eqref{Equ:sub} can be found. The feasible region is a pointed polyhedral cone (left). The hyperplanes where the prediction error at a certain data point equals zero subdivide $P$ into the regions $P(\mathbf{s})$ (middle). Since the objective function is concave in each of these regions, it suffices to check the vertices of all regions (right).}
		\label{Fig:PolyhedralSubdivision}
	\end{figure}
	
	We conclude that it suffices to check all \mbox{$\binom{kn+n}{kd+k} \le (nk)^{O(dk)}$} possible subsets of $kd+k$ equations.
	If the chosen equations are linearly independent, then we can determine the corresponding unique solution and check whether it is a feasible solution to \eqref{Equ:sub}.
	For each chosen set of equations, these steps can be done in $\poly(n,d,k)$ time.
	Among all feasible solutions found that way, we take the best one.
	Consequently, each of the (at most) $2^kn^{dk}$ subproblems can be solved in $(nk)^{O(dk)}\poly(n,d,k)$ time, resulting in the claimed overall running time.
\end{proof}
In comparison to the algorithm for convex loss functions~\cite{ABMM18}, our algorithm for concave loss functions requires more time to solve the $O(2^kn^{dk})$ many subproblems, namely $(nk)^{O(dk)}\poly(n,d,k)$ instead of $\poly(n,d,k)$ time each. This confirms the general theme in optimization that convex problems are easier to solve than non-convex problems. However, due to the combinatorial search, both cases result in an XP overall running time.

\section{Conclusion}

We closed some gaps regarding the computational complexity of training ReLU networks
by proving tight parameterized hardness results and essentially optimal algorithms, thus settling the parameterized complexity.
Notably, as \citeA{GKMR21} point out, every \emph{proper learning} algorithm also solves the training problem. Hence, our results also imply parameterized hardness of proper learning. 

As our results confirm computational intractability also from a parameterized perspective,
this motivates the challenging task to identify suitable parameters to achieve fixed-parameter tractability results.
For example, parameterizing by some ``distance from triviality'' measure (e.g.\ assuming specially structured input data) might be an interesting approach~\cite{Nie06}.
We conclude with some specific open questions:
\begin{itemize}
\item What is the parameterized complexity of \textsc{$k$-ReLU($\ell$)} with respect to~$d$ for~$k\ge 2$ in the zero-error case?
While polynomial-time algorithms for $k=1$ are available (compare \citeA{GKMR21} and \Cref{Sec:Poly}), the zero-error case is NP-hard for $k\geq2$~\cite{BJW19,GKMR21}.
Notably, \citeA[Section~4]{BDL20} showed NP-hardness of the zero-error case for~$k=2$ if the output neuron is also a ReLU.
They give a polynomial-time reduction from the \textsc{2-Hyperplane Separability} problem, which is known to be W[1]-hard with respect to dimension and to have an ETH-based running time lower bound~\cite{GKR09}.
In fact, the reduction of \citeA{BDL20} is a parameterized reduction with respect to~$d$ (the reduction uses two additional dimensions).
Thus, as a corollary, we obtain that, for $k=2$, the zero-error case with ReLU output is W[1]-hard with respect to~$d$ and not solvable in~$n^{o(d)}$ time assuming ETH.

\item Is $1$-\textsc{ReLU($\ell$)} fixed-parameter tractable with respect to~$d$ if all input points contain at most three non-zero entries? For at most four non-zero entries, we showed W[1]-hardness in \Cref{Cor:hard}.

\item Confronting inapproximability results for polynomial-time algorithms (see, e.g., \citeA{GKMR21}) and our W[1]-hardness result for exact algorithms (\Cref{thm:hard}), a natural follow-up question is: Can acceptable worst-case approximation ratios be obtained in FPT-time?

\item What is the (parameterized) complexity of training deeper neural networks (with at least three layers)?
\end{itemize}

\acks{This work was done while Christoph Hertrich was part of the Combinatorial Optimization and Graph Algorithms Group at Technische Universit{\"a}t Berlin and received funding from \mbox{DFG-GRK 2434} ``Facets of Complexity''. Christoph Hertrich would like to thank Amitabh Basu, Marco Di Summa, and Martin Skutella for many valuable discussions about ReLU Neural Networks.
}

\vskip 0.2in
\bibliography{ref}
\bibliographystyle{theapa}

\end{document}